%% file: gfmn_iccv2019.tex
\documentclass[10pt,twocolumn,letterpaper]{article}

\usepackage{iccv}
\usepackage{times}
\usepackage{epsfig}
\usepackage{amsmath}
\usepackage{amssymb}
\usepackage{fixltx2e} 
\usepackage{graphicx} 
\graphicspath{{figures/}}
\usepackage{multirow}
\usepackage{wrapfig}
\usepackage{subcaption}
\usepackage{xspace}
\usepackage{amsthm}

\usepackage[pagebackref=true,breaklinks=true,letterpaper=true,colorlinks,bookmarks=false]{hyperref}

\iccvfinalcopy 
\def\iccvPaperID{5968} 

\newcommand{\Tab}[1]{Tab.~\ref{#1}}
\newcommand{\Fig}[1]{Fig.~\ref{#1}}

\newtheorem{theorem}{Theorem}
\newtheorem{definition}{Definition}

\newtheorem{proposition}{Proposition}
\newtheorem{remark}{Remark}

\ificcvfinal\fi

\begin{document}

\title{Learning Implicit Generative Models by Matching Perceptual Features}

\author{
    Cicero Nogueira dos Santos\thanks{\ Equal contribution.} , Youssef Mroueh\footnotemark[1] , Inkit Padhi\footnotemark[1] , Pierre Dognin\footnotemark[1] \\
    IBM Research, T.J. Watson Research Center, NY  \\
    {\tt\small \{cicerons,mroueh,pdognin\}@us.ibm.com, inkit.padhi@ibm.com}
}


\maketitle

\begin{abstract}
Perceptual features (PFs) have been used with great success in tasks such as transfer learning, style transfer, and super-resolution.
However,
the efficacy of PFs as key source of information for learning generative models is not well studied. 
We investigate here the use of PFs in the context of learning implicit generative models through moment matching (MM).
More specifically, 
we propose a new effective MM approach that learns implicit generative models by performing mean and covariance matching of features extracted from pretrained ConvNets.
Our proposed approach improves upon existing MM methods by: 
(1) breaking away from the problematic min/max game of adversarial learning;
(2) avoiding online learning of kernel functions;
and (3) being efficient with respect to both number of used moments and required minibatch size.
Our experimental results demonstrate that,
due to the expressiveness of PFs from pretrained deep ConvNets, our method achieves state-of-the-art results for challenging benchmarks.
\end{abstract}

\input{introduction} 
\input{approaches}

\input{theory}

\input{related_work}

\input{experiments}

\input{conclusion}

{\small
\bibliographystyle{ieee}
\bibliography{gfmn_iccv2019}
}
\clearpage
\input{appendix}

\end{document}

%% file: introduction.tex
\section{Introduction}
The use of features from deep convolutional neural networks (DCNNs) pretrained on ImageNet \cite{russakovsky:2015} has led to important advances in computer vision.
DCNN features,
usually called \emph{perceptual features (PFs)},
have been used in tasks such as transfer learning \cite{BengioDeepTransfer,ImagenetTransfer}, style transfer \cite{Gatys_2016_CVPR} and super-resolution \cite{Johnson2016Perceptual}.
While there have been previous works on the use of PFs in the context of image generation and transformation \cite{dosovitskiyNIPS2016,Johnson2016Perceptual},
exploration of PFs as key source of information for learning generative models is not well studied.
Particularly, the efficacy of PFs for implicit generative models trained through \emph{moment matching} is an open question.

Moment matching approaches for generative modeling are based on the assumption that one can learn the data distribution by matching the moments of the model distribution to the empirical data distribution.
Two representative methods of this family are based on maximum mean discrepancy (MMD) \cite{gretton:2006MMD,gretton:2012MMD,li:2015:GMMN} and the method of moments (MoM) \cite{RavuriMRV18}.
While MoM based methods embed a probability distribution into a finite-dimensional vector (i.e., matching of a finite number of moments),
MMD based methods embed a distribution into an infinite-dimensional
vector \cite{RavuriMRV18}.
A challenge for MMD methods is to define a kernel function that is statistically efficient and can be used with small minibatch sizes \cite{li:nips2017:MMDGAN}.
A solution comes by using adversarial learning for the online training of kernel functions \cite{li:nips2017:MMDGAN,binkowski2018dMMDGAN}.
However, this solution inherits the problematic min/max game of adversarial learning.
The main challenges of using MoM for training deep generative networks consist in defining millions of sufficiently distinct moments and specifying an objective function to learn the desirable moments.
Ravuri \etal~\cite{RavuriMRV18} addressed these two issues by defining the moments as features and derivatives from a \emph{moment network} that is trained online (together with the generator) by using a specially designed objective function.

In this work
we demonstrate that, by using PFs to perform moment matching,
one can overcome some of the difficulties found in current moment matching approaches.
More specifically,
we propose a simple but effective moment matching method that: (1) breaks away from the problematic min/max game completely;
(2) does not use online learning of kernel functions;
and (3) is very efficient with regard to both number of used moments and required minibatch size.
Our proposed approach, 
named Generative Feature Matching Networks (GFMN),
learns implicit generative models by performing mean and covariance matching of features extracted from all convolutional layers of pretrained deep ConvNets. 
Some interesting properties of GFMNs include:
(a) the loss function is directly correlated to the generated image quality;
(b) mode collapsing is not an issue; and
(c) the same pretrained feature extractor can be used across different datasets.

\begin{figure*}[!th]
\centering
    \includegraphics[width=.9\textwidth]{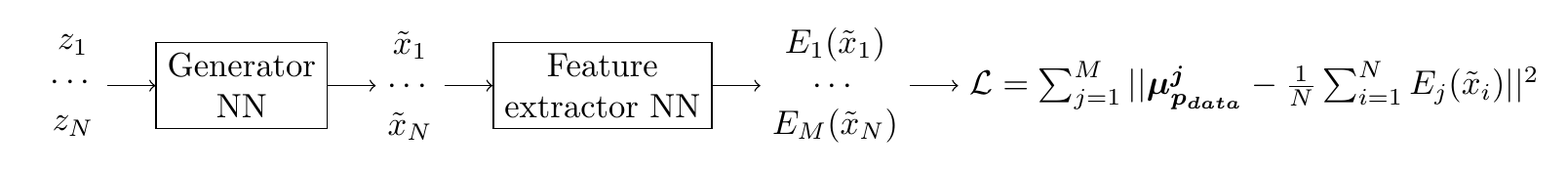}
    \vskip 0in
    \caption{\textbf{GFMN Training}: From $z_1,\dots z_N$ noise signals, generator $G$ creates $N$ images $\tilde{x}_1,\dots \tilde{x}_N$. The fixed pretrained feature extractor $E$ is used to obtain $E_j(\tilde{x}_i)$ features. $\mathcal{L}$ is the $L_2$-norm of the difference between extracted features means of generated and real data, $\boldsymbol{\mu_{p_{data}}^j}$. We precompute $\boldsymbol{\mu_{p_{data}}^j}$ on the entire real dataset (it does not change during training); the mean of generated data is estimated on a minibatch of size $N$. The same strategy is used for variance terms in $\mathcal{L}$.}
    \label{fig:gfmn}
\end{figure*}

We perform an extensive number of experiments with different challenging datasets: CIFAR10, STL10, CelebA and LSUN.
We demonstrate that our approach can achieve state-of-the-art results for challenging benchmarks such as CIFAR10 and STL10.
Moreover,
we show that the same feature extractor is effective across different datasets.
The main contributions of this work can be summarized as follows: 
(1) We propose a new effective moment matching-based approach to train implicit generative models that does not use adversarial or online learning of kernel functions, 
provides stable training, and achieves state-of-the-art results;
(2) We show theoretical results that demonstrate GFMN convergence under the assumption of the universality of perceptual features;
(3) We propose an ADAM-based moving average method that allows effective training with small minibatches;
(4) Our extensive quantitative and qualitative experimental results demonstrate that pretrained autoencoders and DCNN classifiers can be effectively used as (cross-domain) feature extractors for GFMN training.

%% file: approaches.tex
\section{Generative Feature Matching Networks}

\label{secMethods}
\subsection{The method}
Let $G$ be the generator implemented as a neural network with parameters $\theta$, and let $E$ be a pretrained neural network with $L$ hidden layers.
Our proposed approach consists in training $G$ by minimizing the following loss function:
\begin{equation}
\min_{\theta}\sum_{j=1}^M||\mu^j_{p_{data}} -\mu^j_{p_{G}}(\theta)||^2 + ||\sigma^{j}_{p_{data}} -\sigma^{j}_{p_{G}}(\theta) ||^2
\label{eq:lossG}
\end{equation}
where:
\vskip -0.2in
\begin{align*}
\mu^j_{p_{data}}&=\mathbb{E}_{x\sim p_{data}}E_j(x) \in \mathbb{R}^{d_j} \\
\mu^j_{p_{G}}(\theta)&=\mathbb{E}_{z\sim \mathcal{N}(0,I_{n_z})}E_{j}(G(z;\theta)) \in \mathbb{R}^{d_j} \\
\sigma^{j}_{p_{data},\ell}&=\mathbb{E}_{x\sim p_{data}}E_{j,\ell}(x)^2 - [\mu^{j,\ell}_{p_{data}}]^2,\ell=1\dots d_j \\
\sigma^{j}_{p_{G},\ell}(\theta)&=\mathbb{E}_{z\sim\mathcal{N}(0,I_{n_z})}E_{j,\ell}(G(z;\theta))^2\!-\![\mu^{j,\ell}_{p_{G}}]^2,\ell\!=\!1\dots d_j 
\end{align*}
and $||.||^2$ is the $L_2$ loss;
$x$ is a real data point sampled from the data generating distribution $p_{data}$;
$z \in \mathbb{R}^{n_z}$ is a noise vector sampled from the normal distribution $\mathcal{N}(0,I_{n_z})$;
$E_{j}(x)$, denotes the output vector/feature map of the hidden layer $j$ from $E$;
$M \leq L$ is the number of hidden layers used to perform feature matching.
Note that $\sigma^{2}_{p_{data}}$ and $\sigma^{2}_{p_{G}}$ denote the variances of the features from real data and generated data, respectively.
We use diagonal covariance matrices as computing full covariance matrices is impractical for large numbers of features. 

In practice,
we train $G$ by first precomputing estimates of  $\boldsymbol{\mu_{p_{data}}^j}$ and $\boldsymbol{\sigma^{j}_{p_{data}}}$ on the training data,
then running multiple training iterations where we sample a minibatch of generated (\emph{fake}) data and optimize the parameters $\theta$ using stochastic gradient descent (SGD) with backpropagation.
The network $E$ is used for the purpose of feature extraction only and is kept fixed during the training of $G$. Fig.~\ref{fig:gfmn} presents GFMN training pipeline.


\noindent \textbf{Autoencoder Features}:
A natural choice of unsupervised method to train a feature extractor is the autoencoder (AE) framework.
The decoder part of an AE consists exactly of an image generator that uses features extracted by the encoder.
Therefore,
by design,
the encoder network should be a good feature extractor for the purpose of generation.

\noindent \textbf{Classifier Features:}
We experiment with different DCNN architectures pretrained on ImageNet to play the role of the feature extractor $E$.
Our hypothesis is that ImageNet-based PFs are informative enough to allow the training of (cross-domain) generators by feature matching.

\subsection{Matching Feat. with ADAM Moving Average}
\label{sec:adam_mov_avrg}
\textbf{From feature matching loss to moving averages.} In order to train with a mean and covariance feature matching loss, 
one needs large minibatches to obtain good mean and covariance estimates.
With images larger than 32$\times$32, DCNNs produce millions of features, resulting easily in memory issues. 
We propose to alleviate this problem by using \emph{moving averages of the difference of means (covariances) of real and generated data}.
Instead of computing the (memory) expensive feature matching loss in Eq.~\ref{eq:lossG}, 
we keep moving averages $v_j$ of the difference of feature means (covariances) at layer $j$ between  real and generated data.
We detail our moving average strategy for the mean features only, but the same approach applies for the covariances.  
The mean features from the first term of Eq.~\ref{eq:lossG}, $||\boldsymbol{\mu_{p_{data}}^j} \!-\!\mathbb{E}_{z\sim \mathcal{N}(0,I_{n_z})}E_{j}(G(z;\theta)) ||^2$ can be approximated by:
\begin{align*}
v_j^{\top}  \left(\boldsymbol{\mu_{p_{data}}^j} - \frac{1}{N} \sum_{k=1}^N E_{j}(G(z_k;\theta))\right),
\end{align*}
where $N$ is the minibatch size and $v_j$ is a moving average on $\Delta_j$, the difference of the means of the features extracted by the $j$-th layer of $E$:
\begin{align}
\Delta_j&=\boldsymbol{\mu_{p_{data}}^j} - \frac{1}{N} \sum_{k=1}^N E_{j}(G(z_k;\theta)). 
\end{align}
Using these moving averages we replace the first term of the loss given in Eq. \ref{eq:lossG} by
\begin{align}
\min_{\theta} \! \sum_{j=1}^M v_j^{\top} \! \left(\boldsymbol{\mu_{p_{data}}^j} \!-\! \frac{1}{N} \! \sum_{k=1}^N E_{j}(G(z_k;\theta))\right)\!.
\label{eq:lossG_v}
\end{align}
The moving average formulation of features matching above has a major advantage on the naive formulation of Eq.~\ref{eq:lossG} since we can now rely on $v_j$ to get better estimates of the population feature means of real and generated data while using a small minibatch of size $N$. 
For a similar result using the feature matching loss given in Eq.~\ref{eq:lossG}, one would need a minibatch with large size $N$, which is problematic for large number of features.  

\textbf{ADAM moving average: from SGD to ADAM updates.}
Note that for a rate $\alpha$, the moving average $v_j$ has the following update:
\vskip -0.2in
$$v_{j,\mathrm{new}}= (1-\alpha)*v_{j,\mathrm{old}}+\alpha*\Delta_j, \forall j=1\dots M$$
It is easy to see that the moving average is a gradient descent update on the following loss:
\vskip -0.15in
\begin{equation}
\min_{v_j} \frac{1}{2}||v_j-\Delta_j||^2.
\label{eq:loss_moving_average}
\end{equation}
Hence, writing the gradient update with learning rate $\alpha$ we have equivalently:
\begin{equation*}
 v_{j,\mathrm{new}}= v_{j,\mathrm{old}}-\alpha*(v_{j,\mathrm{old}}-\Delta_j)
= (1-\alpha) * v_{j,\mathrm{old}}+\alpha *\Delta_j.
\end{equation*}
With this interpretation of the moving average, we propose to get a better moving average estimate by using the ADAM optimizer \cite{kingma2014adam} on the loss of the moving average given in Eq.~\ref{eq:loss_moving_average}, such that
\vskip -0.15in
$$v_{j,\mathrm{new}}=v_{j,\mathrm{old}}-\alpha ADAM(v_{j,\mathrm{old}}-\Delta_j ).$$
$ADAM(x)$ function is computed as follows:
\begin{align*}
 m_t &=   \beta_1 * m_{t-1} + (1-\beta_1) * x    & \hat{m}_t &=  m_t / (1-\beta_1^t) \\
 u_t &=   \beta_2 * u_{t-1} + (1-\beta_2) * x^2  & \hat{u}_t &=  u_t / (1-\beta_2^t) \\ 
& ADAM(x) =  \hat{m}_t / (\sqrt{\hat{u}_t} + \epsilon),
\end{align*}
where $x$ is the gradient for the loss function in Eq.~\ref{eq:loss_moving_average},
$t$ is the iteration number,
$m_t$ and $u_t$ are the first and second moment vectors at iteration $t$,
$\beta_1\!=\!.9$, $\beta_2\!=\!.999$ and $\epsilon\!=\!10^{-8}$ are constants.
$m_0$ and $u_0$ are initialized as proposed by  \cite{kingma2014adam}.
We refer to \cite{kingma2014adam} for a detailed ADAM optimizer description.

This moving average formulation,
which we call \emph{ADAM Moving Average} (AMA) promotes stable training when using small minibatches. 
Although we detail AMA using mean feature matching only, we use this approach for both mean and covariance matching.
The main advantage of AMA over simple moving average (MA) is in its adaptive first and second order moments that ensure stable estimation of the moving averages $v_j$. In fact, this is a non-stationary estimation since the mean of the generated data changes in the training, and it is well known that ADAM works well for such online non-stationary losses \cite{kingma2014adam}. 

In Section \ref{sec:StabilityAdamMovingAverage} we provide experimental results supporting: (1) The memory advantage that the AMA formulation of feature matching offers over the naive implementation;
(2) The stability advantage and improved generation results that AMA allows compared to the naive implementation.
We discuss in Appendix~2 the advantage of AMA on MA from a regret bounds point of view \cite{reddi_convergence_2018}.


%% file: theory.tex
\section{Universality of PFs and GFMN Convergence}
\label{sec:theory}
Our proposed approach is related to the recent body of work on MMD or MM based generative models \cite{li:2015:GMMN,dziugaite:2015MMD_net,li:nips2017:MMDGAN,binkowski2018dMMDGAN,RavuriMRV18}. We highlight  the main differences between MMD-GANs  and GFMN in terms of requirements on the kernel for MMD-GAN and on the feature map (Extractor) for GFMN, that ensure convergence of the generator to the data distribution.  See Tab.~\ref{tab:compa} for a summary.\\

\noindent\textbf{GMMN, MMD-GAN Convergence: MMD Matching with Universal Kernels.}
We start by reviewing known results on MMD. Let $\mathcal{H}_{k}$ be a Reproducing Kernel Hilbert Space (RKHS) defined with a continuous kernel $k$. Informally, \emph{$k$ is universal if  any bounded continuous function can be approximated to an arbitrary precision  in $\mathcal{H}_{k}$} (formal definition in Appendix ). Theorem \ref{theo:mmd} \cite{gretton:2012MMD}  shows that the MMD is a well defined  metric for universal kernels. 
\vskip -0.3 in 
\begin{theorem}[\cite{gretton:2012MMD}] 
Given a kernel $k$, let $p,q$ be two distributions, their MMD  is:
$\text{MMD}^2(k,p,q)= \left|\left|\mu_p-\mu_q\right|\right|^2_{\mathcal{H}_k},$ where $\mu_p=\mathbb{E}_{x\sim p}k_{x}$ is the  mean embedding. 
 If $k$ is universal then $\text{MMD}^2(k,p,q)=0$ if and only if $p=q$. 
\label{theo:mmd}
\end{theorem}

\begin{table*}[ht!] 
\small
\begin{center}
\begin{tabular}{| l | l | l | l | l | }
\hline
 & \bf Metric  & \bf Kernel/Feature Map      &\bf~~Generative M.~~ \\
  &  & \bf Convergence Conditions      &\bf~~Optimization ~~ \\
\hline
GMMN \cite{li:2015:GMMN,dziugaite:2015MMD_net} & $\text{MMD}(k,p,q)$ & Universal $k$   & $\min$ prob. \\
\hline
$\text{MMD}$-GAN & $D_{\text{MMD}}(p,q)$ &  $k\circ \phi$   & $\min/\max$ prob. \\
 \cite{li:nips2017:MMDGAN,binkowski2018dMMDGAN,RavuriMRV18}&  &  $k$ Fixed universal \& lipschitz   & \\
  &  &  $\phi$  lipschitz learned   & \\
\hline
GFMN & $\text{MMD}(K_{\Phi},p,q)$ &  Universal Feature Map $\Phi$   & $\min$ prob.  \\
(This work) &  &  $\Phi$  fixed  &$\mu_{p_{data}}$ precomputed  \\
\hline
\end{tabular}

\caption{Comparison of different approaches using MMD matching for implicit generative modeling. GFMN has two practical computational advantages it avoids the min/max game and allows to use a pre-computed mean embedding on the real data. Theoretically GFMN converges to the real data distribution if the feature extractor used was universal (See text for definition of Universal features as given in  \cite{micchelli:2006_uk} ) . }
\label{tab:compa} 
\end{center}
\vskip -0.4in
\end{table*}  
Given a Universal kernel such as a Gaussian Kernel as outlined in GMMN \cite{li:2015:GMMN,dziugaite:2015MMD_net}, one can learn implicit Generative models $G_{\theta}$ that defines a family of distribution $\{q_{\theta}\}$ by minimizing the MMD distance:
\vskip -0.1in
\begin{equation}
\inf_{\theta} \text{MMD}(k,p_{data},q_{\theta})
\label{eq:MMDmin}
\end{equation}
\vskip -0.1in
Assuming $p_{data}$ is in the family $\{q_{\theta}\}$ $\left(\exists \theta^*, q_{\theta^*}=p_{data}\right)$, the infimum of MMD minimization for a universal kernel is achieved for $q_{\theta}\! = \! p_{data}$ (immediate consequence of Theorem~1).
This elegant setup for MMD matching with universal kernels, while avoiding the difficult min/max game in GAN, does not translate into good results in image generation. 
To remedy that, other discrepancies introduced in \cite{li:nips2017:MMDGAN,binkowski2018dMMDGAN,RavuriMRV18} compose universal kernels $k$ with a feature map $\phi \in \Psi$ as follows:
$D_{\text{MMD}}(p,q)=\sup_{\phi \in \Psi} \text{MMD}(k \circ \phi,p,q). $ For learning implicit generative models \cite{li:nips2017:MMDGAN} replaces $\text{MMD}$ in Eq.~\eqref{eq:MMDmin} by $D_{\text{MMD}}$. 
Under  conditions on the kernel and the learned feature map this discrepancy is continuous in the weak topology (Prop.~2 in \cite{arbel2018gradient,li:nips2017:MMDGAN}). Nevertheless,  learning generative models   remains challenging  with it as it boils down to a min/max game as in  original GAN  \cite{goodfellow2014generative}.\\

\noindent \textbf{GFMN Convergence: MMD  Matching with Universal Features.} 
While universality is usually thought on the kernel level, it is not straightforward to define universality for kernels defined by feature maps.
Micchelli \etal \cite{micchelli:2006_uk} define universality of feature maps and how it connects to their corresponding kernels. 
Specifically for a fixed feature set on a space $\mathcal{X}$ (space of images)  $S=\{\phi_j, j \in I,\phi_j: \mathcal{X}\to \mathbb{R}\}$, where $I$ is a countable set of indices, define the  kernel
$K_{\phi}(x,y)=\sum_{j\in I} \phi_j(x)\phi_j(y)$. Micchelli \etal \cite{micchelli:2006_uk} in Thm.~7 show that this Kernel is universal if the set $S$ is universal. 
Informally speaking, \emph{a feature set $S$ is universal if linear functions in this feature space  $(\sum_{j\in I}u_j\phi_j(x))$, are dense in the set of continuous bounded functions} (formal definition in Appendix 1). 

This is of interest since GFMN corresponds to MMD matching with a kernel $K_{\Phi}$ defined on a fixed feature map $\Phi(x) \!=\! \{\phi_j(x)\}_{ j \in I}$, where $I$ is finite. 
We have $K_{\Phi}(x,y) \! = \! \langle \Phi(x),\Phi(y)\rangle \! = \! \sum_{j\in I} \phi_j(x)\phi_j(y)$ and 
$$\text{MMD}^2(K_{\Phi},p,q)= \left|\left|\mathbb{E}_{x\sim p} \Phi(x)- \mathbb{E}_{x\sim q}\Phi(x) \right|\right|^2.$$
For $\text{MMD}^2(K_{\Phi},p,q)$ to be a metric it is enough to have the set features $S$  be universal (by Thm.\ref{theo:mmd} and Thm.~7 in \cite{micchelli:2006_uk}). 
Prop.~\ref{pro:UnivFeatures} gives conditions for  GFMN convergence:
\begin{proposition}
 Assume $p_{data}$ belongs to the family defined by the generator $\{q_{\theta}\}_{\theta}$. GFMN converges to  the real distribution by matching in a feature space $S=\{\phi_j, j \in I \}$, where I is a countable set, if the features set S is universal (informally means that any continuous functions can be written as linear combination in the span of $S$) .
 \label{pro:UnivFeatures}
\end{proposition}
\begin{proof}
 $S$ is universal $\implies k_{\Phi}$ is universal \cite{micchelli:2006_uk}. Hence $\text{MMD}(k_{\Phi},p_{data},q_{\theta})\!=\!0$ iff $q_{\theta}\!=\!p_{data}$. 
 GFMN solves $\inf_{\theta}\text{MMD}^2(K_{\Phi},p_{data},q_{\theta})$, and the infimum is achieved  for $\theta$ such that $q_{\theta}=p_{data}$ (  $p_{data} \in\{q_{\theta}\}_{\theta}$ ).
 \vskip -0.2in
\end{proof}

\begin{remark}The analysis covers here mean matching but the same applies to covariance matching considering $S=\{\phi_j,\phi_j\phi_k, j,k \in I\}$.
\end{remark}

\noindent \textbf{Universality of Perceptual Features in Computer Vision.} From Prop. \ref{pro:UnivFeatures} we see that for GFMN to be convergent with pretrained feature extractors $E_j$ that are \textbf{perceptual features} (such as features from VGG or ResNet pretrained on ImageNet), we need to assume universality of those features in the image domain. 
We know from transfer learning that features from ImageNet pretrained VGG/ResNet can express any functions for a downstream task by finding a linear weight in their span. Note that this is the definition of universal feature as given in \cite{micchelli:2006_uk}: continuous functions can be approximated in the linear span of those features.
Hence, assuming universality of PFs defined by ImageNet pretrained VGG or ResNet, GFMN is guaranteed to converge to the data distribution by Prop.~\ref{pro:UnivFeatures}. 
Our results complement the common wisdom on ``universality'' of PFs in  transfer learning and style transfer by showing that they are sufficient for learning implicit generative models.

%% file: related_work.tex
\section{Related work}
\label{sec:relatedwork}


GFMN is related to the recent body of work on MMD and  moment matching based generative models \cite{li:2015:GMMN,dziugaite:2015MMD_net,li:nips2017:MMDGAN,binkowski2018dMMDGAN,RavuriMRV18}.
The closest to our method is the Generative Moment Matching Network + Autoencoder (GMMN+AE) proposed in \cite{li:2015:GMMN}.
In GMMN+AE,
the objective is to train a generator $G$ that maps from a prior uniform distribution to the latent code learned by a pretrained AE, and then uses the frozen pretrained decoder to map back to image space.
As discussed in Section \ref{sec:theory} one key difference in our approach is that,
while GMMN+AE uses a Gaussian kernel to perform moment matching using the AE low dimensional latent code, GFMN performs mean and covariance matching in a PF space induced by a non-linear kernel function (a DCNN) that is orders of magnitude larger than the AE latent code, and that we argued is universal in the image domain.

Li \etal~\cite{li:nips2017:MMDGAN} demonstrate that GMMN+AE is not competitive with GANs for challenging datasets such as CIFAR10. MMD-GANs,  discussed in Section \ref{sec:theory},
 demonstrated competitive results with the use of adversarial learning by learning a feature map in conjuction with a Gaussian kernel  \cite{li:nips2017:MMDGAN,binkowski2018dMMDGAN}.
Finally,
Ravuri \etal~\cite{RavuriMRV18} recently proposed a method to perform online learning of the moments while training the generator.
Our proposed method differs by using fixed pretrained PF extractors for moment matching.

Bojanowski \etal~\cite{bojanowski2018optimizing} proposed the Generative Latent Optimization (GLO) model that jointly optimizes model parameters and noise input vectors $z$,
while avoiding adversarial training.
Our work relates also to plug and play generative models of \cite{NguyenCBDY17} where a pretrained classifier is used to sample new images, using MCMC sampling methods. 

Our work is also related to AE-based generative models variational AE (VAE) \cite{kingma2013auto}, adversarial AE (AAE) \cite{makhzani2016} and Wasserstein AE (WAE) \cite{tolstikhin2018ICLR}.
However, GFMN is quite distinct from these methods because it uses  pretrained AEs to play the role of feature extractors only, while these methods aim to impose a prior distrib. on the latent space of AEs.
Another recent line of work that involves the use of AEs in generative models consists in applying AEs to improve GANs stability \cite{zhaoICLR2017,warde2017iclr}.
Finally, our objective function is related to the McGan loss function \cite{mroueh17_mcgan}, where authors match first and second order moments.



%% file: experiments.tex
\vskip -0.2in
\section{Experiments}
\label{sec:experiments}
\vskip -0.07in
\subsection{Experimental Setup}
\vskip -0.07in
\noindent \textbf{Datasets:}
We evaluate our proposed approach on images from CIFAR10 \cite{cifar10} (50k train., 10k test, 10 classes), STL10 \cite{stl10} (5k train., 8k test, 100k unlabeled, 10 classes), CelebA \cite{celeba} (200k) and LSUN bedrooms \cite{yu15lsun} datasets.
STL10 images are rescaled to 32$\times$32,
while CelebA and LSUN images are rescaled to either 64$\times$64 or 128$\times$128, 
depending on the experiment.
CelebA images are center-cropped to 160$\times$160 before rescaling.

\noindent \textbf{GFMN Generator:}
In most of our experiments the generator $G$ uses a DCGAN-like architecture \cite{radfordMC15}.
For CIFAR10, STL10, LSUN and CelebA$_{64\times64}$,
we use two extra layers as commonly used in previous works \cite{mrouehS17fishergan,GulrajaniWGANGP17}.
For CelebA$_{128\times128}$ and some experiments with CIFAR10 and STL10,
we use a ResNet-based generator such as the one in \cite{GulrajaniWGANGP17}.
Architecture details are in the supplementary material.

\noindent \textbf{Autoencoder Features}:
For most AE experiments, 
we use an encoder network whose architecture is similar to the discriminator in DCGAN (strided convolutions).
We use batch normalization and ReLU non-linearity after each convolution.
We set the latent code size to 128, 128, and 512 for CIFAR10, STL10 and CelebA, respectively.
To perform feature extraction,
we get the output of each ReLU in the network.
Additionally,
we also perform some experiments where the encoder uses a VGG19 architecture.
The decoder network $D$ uses a network architecture similar to our generator $G$.
More details in the supplementary material.

\noindent \textbf{Classifier Features:}
We perform our experiments on classifier features with VGG19 \cite{Simonyan14c} and Resnet18 networks \cite{He2016} which we pretrained using the whole ImageNet dataset \cite{russakovsky:2015} with 1000 classes.
Pretrained ImageNet classifiers details can be found in the supplementary material.

\vskip -0.01in
\noindent \textbf{GFMN Training:}
GFMNs are trained with an ADAM optimizer; most hyperparameters are kept fixed across datasets. We use $n_z\!=\!100$ and minibatch of 64. Dataset dependent learning rates are used for updating $G$ ($10^{-4}$ or $5\!\times\!10^{-5}$) and AMA ($5\!\times\!10^{-5}$ or $10^{-5}$).
We use AMA moving average (Sec. \ref{sec:adam_mov_avrg}) in all reported experiments.

\begin{table*}[ht!]
\small
  \caption{CIFAR10 results for GFMN with different feature extractors.}
  \label{tab:is_fid_gfmn}
  \centering
  \begin{tabular}{lcccccc}
    \hline
    \bf $E$ Type & \bf $E$ Arch. & \bf Pre-trained On & \bf \# features & \bf $G$ Arch.  & \bf IS & \bf FID (5K/50K) \\
    \hline
    Encoder & DCGAN & CIFAR10 & 60K & DCGAN & 4.51  $\pm$ 0.06 & 82.8 / 78.3  \\
    Encoder & VGG19 & ImageNet & 296K & DCGAN & 4.95 $\pm$ 0.06 & 61.6 / 57.2 \\
    Classifier & Resnet18 & ImageNet & 544K  & DCGAN & 7.92 $\pm$ 0.10 & 29.1 / 24.3 \\
    Classifier & VGG19  & ImageNet  & 296K & DCGAN & 7.88 $\pm$ 0.08 & 25.5 / 20.8   \\
    Classifier & VGG19 + Resnet18  & ImageNet & 832K & DCGAN & 8.08 $\pm$ 0.08 & 25.5 / 20.9 \\
    Classifier & VGG19 + Resnet18  & ImageNet & 832K & Resnet & \bf 8.27 $\pm$ 0.09 & \bf 18.1 / 13.5 \\
    \hline
  \end{tabular}
\end{table*}

\begin{figure*}[!ht]
    \centering
    \begin{subfigure}[b]{0.27\textwidth}
        \includegraphics[width=\textwidth]  {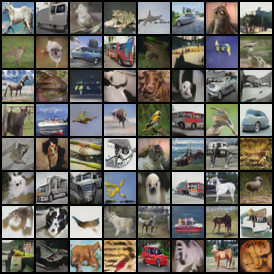}
        \caption{CIFAR10 }
        \label{fig:cifar_vgg_resnet}
    \end{subfigure}
    \begin{subfigure}[b]{0.27\textwidth}
        \includegraphics[width=\textwidth] {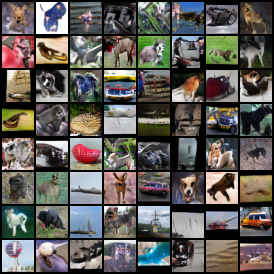}
        \caption{STL10}
        \label{fig:stl_vgg_resnet}
    \end{subfigure}
    \begin{subfigure}[b]{0.27\textwidth}
        \includegraphics[width=\textwidth]  {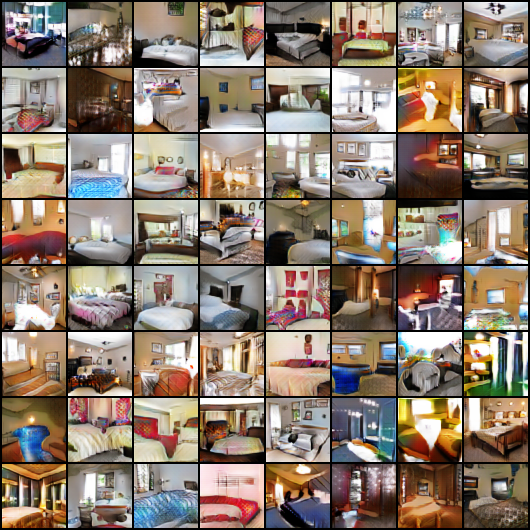}
        \caption{LSUN}
        \label{fig:lsun_vgg}
    \end{subfigure}
\caption{Generated samples from GFMN that uses as feature extractor VGG-19+Resnet18 (\ref{fig:cifar_vgg_resnet}, \ref{fig:stl_vgg_resnet}) and VGG-19 net (\ref{fig:lsun_vgg}).}
    \label{fig:gfmn-class}
\end{figure*}

\subsection{Autoencoder Features vs. (Cross-domain) Classifier Features}
\label{sec:ae_features}
This section presents a comparative study on the use of pretrained autoencoders and cross-domain classifiers as feature extractors in GFMN.
\Tab{tab:is_fid_gfmn} shows the Inception Score (IS) \cite{salimans2016improved} and Fr\'{e}chet Inception Distance (FID) \cite{heuselRUNKH17FID} for GFMN trained on CIFAR10 using different feature extractors $E$.
The two first rows in \Tab{tab:is_fid_gfmn} correspond to GFMN models that use pretrained encoders as $E$,
while the last four rows use pretrained VGG19/Resnet18 ImageNet classifiers.
We can see in \Tab{tab:is_fid_gfmn} that there is a large boost in performance when ImageNet classifiers are used as feature extractors instead of encoders.
Despite the classifiers being trained on a different domain (ImageNet vs. CIFAR10), 
the classifier features are significantly more effective.
While the best IS with encoders is 4.95,
the lowest IS with ImageNet classifier is 7.88.
Additionally,
when using simultaneously VGG19 and Resnet18 as feature extractors (two last rows), which increases the number of features to 832K,
we get even better performance.
Finally,
we achieve the best performance in terms of both IS and FID (last row\footnote{Average result of five runs with different random seeds.}) when using a generator architecture that contains residual blocks, similar to the one propose in \cite{GulrajaniWGANGP17}.

Random samples from GFMN$^{\text{VGG19+Resnet18}}$ trained with CIFAR10 and STL10 are shown in Figs.~\ref{fig:cifar_vgg_resnet} and \ref{fig:stl_vgg_resnet} respectively.
Fig.~\ref{fig:lsun_vgg} shows random samples from GFMN$^{\text{VGG19}}$ trained with LSUN bedrooms dataset (resolution $64\!\times\!64$).
Fig. \ref{fig:celeba_sz128} presents samples from GFMN$^{\text{VGG19}}$ trained with CelebA dataset with resolution $128\!\times\!128$, 
which shows that GFMN can achieve good performance with image resolutions larger than $32\!\times\!32$.
These results also demonstrate that: (1) the same classifier (VGG19 trained on ImageNet) can be successfully applied to train GFMN models across different domains; (2) perceptual features from DCNNs encapsulate enough statistics to allow the learning of good generative models through moment matching.



\begin{figure}[!h]
\centering
    \includegraphics[width=0.27\textwidth,]{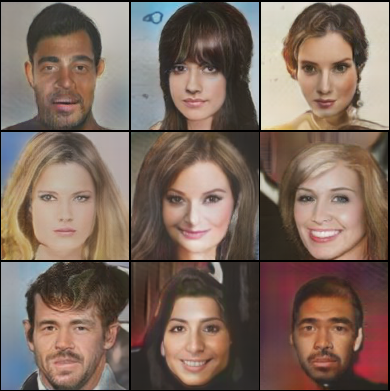}
    \caption{Samples from GFMN$^{\text{VGG19}}$ trained on CelebA with image size $128\times128$.}
    \label{fig:celeba_sz128}
\end{figure}

\Tab{tab:is_fid_numlayers} shows IS and FID for increasing number of layers (i.e. number of features) in our extractor VGG19. 
We select up to 16 layers, excluding the output of fully connected layers. 
Using more layers dramatically improves the performance of the feature extractor, reaching IS and FID peak performance when the maximum number of layers is used.
Note that the features are ReLU activation outputs, meaning the encodings may be quite sparse.
In Appendix 7 we show qualitative results that corroborate these results.


To verify whether the number of features is the main factor for performance,
we conducted an experiment where we train an AE with an encoder using a VGG19 architecture.
This encoder is pretrained on ImageNet and produces a total of 296K features.
The second row in \Tab{tab:is_fid_gfmn} shows the results for this experiment.
Although there is improvement in both IS and FID compared to the DCGAN encoder (first row),
the boost is not comparable to the one obtained with a VGG19 classifier.
In other words, features from classifiers are significantly more informative than AEs features for the purpose of training generators by feature matching.

\begin{table}[!ht]
\small
  \caption{Impact of the number of layers/features used for feature matching in GFMN (1K=$2^{10}$).}
  \label{tab:is_fid_numlayers}
  \centering
  \begin{tabular}{lccc}
    \hline
     \bf \# layers    & \bf \# features & IS & FID (5K / 50K) \\
    \hline
    1     & 64K  & 4.68 $\pm$ 0.05 & 118.6 / 114.8\\
    3     & 160K & 5.59 $\pm$ 0.08 & 83.2 / 78.2\\
    5     & 208K & 6.12 $\pm$ 0.05 & 53.8 / 49.3\\
    7     & 240K & 6.99 $\pm$ 0.06 & 39.4 / 34.9 \\
    9     & 264K & 7.26 $\pm$ 0.06 & 32.3 / 27.7 \\
    11    & 280K & 7.72 $\pm$ 0.08 & 29.6 / 25.0 \\
    13    & 290K & 7.49 $\pm$ 0.09 & 29.2 / 24.8 \\
    15    & 294K & 7.62 $\pm$ 0.04 & 27.6 / 22.7 \\
    16    & 296K & \bf 7.88 $\pm$ 0.08 & \bf 25.5 / 20.8 \\
    \hline
  \end{tabular}
\vskip -0.1in
\end{table}

\subsection{AMA and Training Stability} \label{sec:StabilityAdamMovingAverage}

\begin{figure*}[!ht]
    \centering
    \begin{subfigure}[b]{0.27\textwidth}
        \includegraphics[width=\textwidth]  {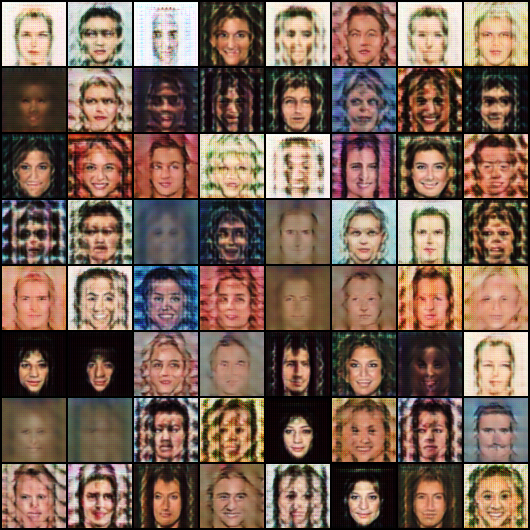}
        \caption{MA - mbs 64}
        \label{fig:celeba_ae_ma_bs64}
    \end{subfigure}
    \begin{subfigure}[b]{0.27\textwidth}
        \includegraphics[width=\textwidth]  {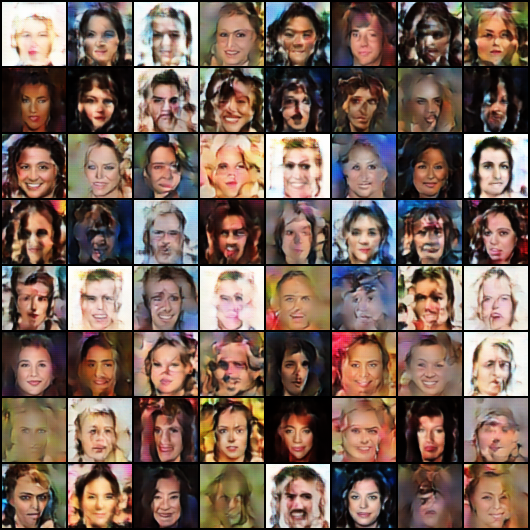}
        \caption{MA - mbs 512}
        \label{fig:celeba_ae_ma_bs512}
    \end{subfigure}
    \begin{subfigure}[b]{0.27\textwidth}
        \includegraphics[width=\textwidth]  {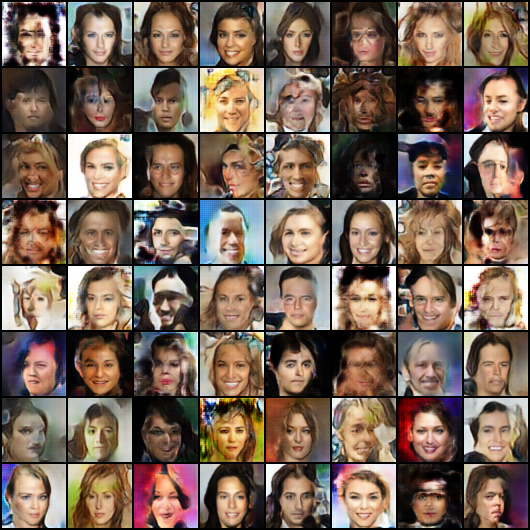}
        \caption{AMA - mbs 64}
        \label{fig:celeba_ae_ama_bs64}
    \end{subfigure}
\caption{Generated images from GFMN trained with either simple Moving Average (MA) (\ref{fig:celeba_ae_ma_bs64} and \ref{fig:celeba_ae_ma_bs512}) or Adam Moving Average (AMA) (\ref{fig:celeba_ae_ama_bs64}), and various minibatch sizes (mbs). While small minibatch sizes have a big negative effect for MA, it is not an issue for AMA.}
\label{fig:celeba_diff_batchsize}
\vskip -0.15in
\end{figure*}

This section presents experimental results that evidence the advantage of our proposed ADAM moving average (AMA) over the simple moving average (MA).
The main benefit of AMA is the promotion of stable training when using small minibatches.
The ability to train with small minibatches is \emph{essential} due to GFMN's need for large number of features from DCNNs, 
which becomes a challenge in terms of GPU memory usage.
Our Pytorch \cite{paszke2017automatic} implementation of GFMN can only handle minibatches of size up to 160 when using VGG19 as a feature extractor and image size $64\!\times\!64$ on a Tesla K40 GPU w/ 12GB of memory.
A more optimized implementation minimizing memory overhead could, in principle, handle somewhat larger minibatch sizes (as could a more recent Tesla V100 w/ 16 GB). 
However, increase image size or feature extractor size and the memory footprint increases quickly. We will always run out of memory when using larger minibatches, regardless of implementation or hardware.

All experiments in this section use CelebA training set, and a feature extractor using the encoder from an AE following a DCGAN-like architecture.
This feature extractor is smaller than VGG19/Resnet18 allowing for minibatches of size up to 512 for image size $64\!\times\!64$.  
Fig.~\ref{fig:celeba_diff_batchsize} shows generated images from GFMN trained with either MA or our proposed AMA.
For MA, generated images from GFMN trained with 64 and 512 minibatch size are presented in Figs. \ref{fig:celeba_ae_ma_bs64} and \ref{fig:celeba_ae_ma_bs512} respectively.
For AMA, Fig. \ref{fig:celeba_ae_ama_bs64} shows results for minibatch size 64.
In MA training, the minibatch size has a \emph{tremendous} impact on the quality of generated images:
with minibatches smaller than 512, almost all images generated are quite distorted.
On the other hand,
when using AMA, 
GFMN generates much better images with minibatch size 64 (Fig. \ref{fig:celeba_ae_ama_bs64}).
For AMA,
increasing the minibatch size from 64 to 512 does not improve the quality of generated images for the given dataset and feature extractor.
In the supplementary material,
we show a comparison between MA and AMA with VGG19 ImageNet classifier as feature extractor for a minibatch size of 64.
AMA also displays a very positive effect on the quality of generated images when a stronger feature extractor is used.
An alternative for training with larger minibatches would be the use of multi-GPU, multi-node setups.
However, performing large scale experiments is beyond the scope of the current work.
Moreover, many practitioners do not have access to a GPU cluster, and the development of methods that can also work on a single GPU with small memory footprint is essential.
\begin{figure}[!ht]
\centering
    \includegraphics[width=0.32\textwidth,]{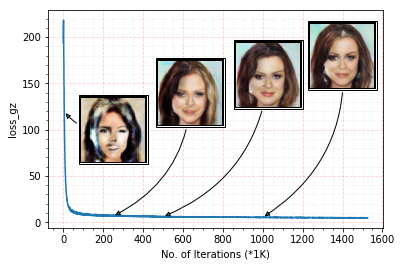}
    \caption{Loss as a function of training epochs with example of generated faces.}
    \label{fig:loss}
\vskip -0.1in
\end{figure}

An important advantage of GFMN over adversarial methods is its training stability.
\Fig{fig:loss} shows the evolution of the generator loss per epoch and generated examples when using AMA.
There is a clear correlation between the quality of generated images and the loss. 
Moreover, mode collapsing was not observed in our experiments with AMA.

\subsection{Comparison to the State-of-the-art}
\begin{table*}[!ht]
\small
  \caption{Inception Score and FID of different generative models for CIFAR10 and STL10.}
  \label{tab:is_fid_sota}
  \centering
  \begin{tabular}{lcccc}
    \hline
    \bf Model & \multicolumn{2}{c}{\bf CIFAR 10} & \multicolumn{2}{c}{\bf STL 10} \\
        & \bf IS & \bf FID (5K / 50K) & \bf IS & \bf FID (5K / 50K) \\
    \hline
    Real data &  11.24$\pm$.12 & 7.8 / 3.2 & 26.08$\pm$.26 & 8.08 / 4.0 \\
    \hline
    \multicolumn{5}{c}{\bf No Adversarial Training}\\
    \hline
    GMMN    \cite{li:nips2017:MMDGAN}  & 3.47$\pm$.03 &  \\
    GMMN+AE \cite{li:nips2017:MMDGAN}  & 3.94$\pm$.04 &  \\
    (ours) GFMN$^{\text{VGG+Resnet}}$ & 8.08 $\pm$ 0.08 & 25.5 / 20.9  & 8.57 $\pm$ 0.08 & 34.2 / 17.2 \\
    (ours) GFMN$^{\text{VGG+Resnet}}$ (Resnet G) & \bf 8.27 $\pm$ 0.09 & \bf 18.1 / 13.5  & \bf 9.12 $\pm$ 0.09 & \bf 31.6 / 13.9 \\
    \hline
    \multicolumn{5}{c}{\bf Adversarial Training \& Online Moment Learning Methods (Unsupervised)}\\
    \hline
    MMD GAN \cite{li:nips2017:MMDGAN} & 6.17$\pm$.07 &  \\
    MMD$_{rq}$ GAN \cite{binkowski2018dMMDGAN} & 6.51$\pm$.03 & 39.9 / -\\
    WGAN-GP \cite{miyato2018spectral} & 6.68$\pm$.06 & 40.2 / -  & 8.42$\pm$.13 & 55.1 / -\\
    SN-GANs \cite{miyato2018spectral} & 7.58$\pm$.12 & 25.5 / - & 8.79$\pm$.14 & 43.2 / -  \\
    MoLM-1024 \cite{RavuriMRV18} & 7.55$\pm$.08 & 25.0 / 20.3 & \\
    GAN-DFM \cite{warde2017iclr}  & 7.72$\pm$.13 \\
    MoLM-1536 \cite{RavuriMRV18} & 7.90$\pm$.10 & 23.3 / 18.9 & \\
    \hline
    \multicolumn{5}{c}{\bf Adversarial Training (Supervised)}\\
    \hline
    Impr. GAN \cite{salimans2016improved} & 8.09$\pm$.07 &  \\
    FisherGAN (Resnet G) \cite{mrouehS17fishergan} & 8.16$\pm$.12 &\\
    WGAN-GP (Resnet G) \cite{GulrajaniWGANGP17} & 8.42$\pm$.10 &  \\
    \hline
  \end{tabular}
\vskip -0.15in
\end{table*}

In \Tab{tab:is_fid_sota},
we compare GFMN results with different adversarial and non-adversarial approaches for CIFAR10 and STL10.
In the middle part of the table,
we report results for recent unsupervised models that use a DCGAN-like architecture in the generator.
Despite using a frozen cross-domain feature extractor,
GFMN outperforms the unsupervised systems in IS and FID for both datasets.
The bottom part of \Tab{tab:is_fid_sota} includes results for supervised approaches. Some of these models use a Resnet architecture in the generator as indicated in parenthesis.
Note that GAN-based methods that perform conditional generation use direct feedback from the labels in the form of log likelihoods from the discriminator (e.g. using the $k\!+\!1$ trick from \cite{salimans2016improved}). 
In contrast, our generator is trained with a loss function that \emph{only} performs feature matching. Our generator is agnostic to the labels and there is no feedback in the form of a log likelihood from the labeled data.
Despite that,
GFMN produces results that are at the same level of supervised GAN models that use labels from the target dataset.


We performed additional experiments with a WGAN-GP architecture where: 
(1) the discriminator is a VGG19 or a Resnet18; 
(2) the discriminator is pretrained on ImageNet.
The goal was to evaluate if WGAN-GP can benefit from DCNN classifiers pretrained on ImageNet.
Although we tried different hyperparameter combinations,
we were not able to successfully train WGAN-GP with VGG19 or Resnet18 discriminators (details in Appendix 8).

%% file: conclusion.tex
\section{Discussion \& Concluding Remarks}
\label{sec:conclusion}
We achieve successful non-adversarial training of implicit generative models by introducing different key ingredients:
(1) moment matching on perceptual features from all layers of pretrained neural networks;
(2) a more robust way to compute the moving average of the mean features by using ADAM optimizer, which allows us to use small minibatches;
and (3) the use of perceptual features from multiple neural networks at the same time (VGG19 + Resnet18).

Our quantitative results in Tab.~\ref{tab:is_fid_sota} show that GFMN achieves better or similar results compared to the state-of-the-art Spectral GAN (SN-GAN) \cite{miyato2018spectral} for both CIFAR10 and STL10.
This is an impressive result for a non-adversarial feature matching-based approach that uses pretrained cross-domain feature extractors and has stable training.
When compared to MMD approaches \cite{li:2015:GMMN,dziugaite:2015MMD_net,li:nips2017:MMDGAN,binkowski2018dMMDGAN,RavuriMRV18},
GFMN presents important distinctions (some of them already listed in Secs. \ref{sec:theory} and \ref{sec:relatedwork}) which make it an attractive alternative.
Compared to GMMN and GMMN+AE \cite{li:2015:GMMN},
we can see in Tab.~\ref{tab:is_fid_sota} that GFMN achieves far better results.
In the supplementary material,
we also show a qualitative comparison between GFMN and GMMN results.
Compared to recent adversarial MMD methods (MMD GAN) \cite{li:nips2017:MMDGAN,binkowski2018dMMDGAN}
GFMN also presents significantly better results while avoiding the problematic min/max game.
GFMN achieves better results than the Method of Learned Moments (MoLM)
\cite{RavuriMRV18}, 
while using a much smaller number of features to perform matching.
The best performing model from \cite{RavuriMRV18}, MoLM-1536, uses around 42 million moments to train the CIFAR10 generator,
while our best GFMN model uses around 850K moments/features only,
almost 50x less. 

One may argue that the best GFMN results are obtained with feature extractors trained with classifiers.
However, there are two important points to note:
(1) we use a cross domain feature extractor and do not use labels from the target datasets (CIFAR10, STL10, LSUN, CelebA);
(2) classifier accuracy does not seem to be the most important factor for generating good features: VGG19 classifier produces features as good as the ones from Resnet18, although the former is less accurate (more details in supplementary material). We are confident that GFMN can achieve state-of-the-art results with features from classifiers trained with unsupervised methods such as \cite{caron2018deepclustering}.

In conclusion, this work presents important theoretical and practical contributions that shed light on the effectiveness of perceptual features for training implicit generative models through moment matching.

%% file: appendix.tex
\section{Appendix}

\subsection{Continuation of Universality of PFs and GFMN Convergence}
We summarize here the main definitions and theorems from \cite{micchelli:2006_uk} regarding universality of kernels and feature maps.\\

\noindent \textbf{Universal Kernels.} The following defines a universal kernel
\begin{definition}[Universal Kernel]
Given a kernel $K$ defined on $\mathcal{X}\times \mathcal{X}$. Let $\mathcal{Z}$ be any compact subset of $\mathcal{X}$. Define the space of kernel sections:
$$\mathcal{K}(\mathcal{Z})= \overline{\text{span}}\{K_y, y \in \mathcal{Z}\},$$
where $K_y: \mathcal{X}\to \mathbb{R}$, $K_{y}(x)=K(x,y)$.
Let $\mathcal{C}(Z)$ be the space of all continuous real valued functions defined on $\mathcal{Z}$.
A kernel is said \textbf{universal} if for any choice of $\mathcal{Z}$ (compact subset of $\mathcal{X}$) $\mathcal{K}(\mathcal{Z})$ is dense in $\mathcal{C}(\mathcal{Z})$. 
\end{definition}
In other words a kernel is universal if $\mathcal{C}(\mathcal{Z})= K(\mathcal{Z}).$  Meaning if  any continuous function can be expressed in the span of $K_{y}$.\\

\noindent \textbf{Universal Feature Maps.}We turn now for kernels defined by feature maps and how to characterize their universality. Consider a continuous feature map $\Phi: \mathcal{X}\to \mathcal{W}$, where $(\mathcal{W},\langle, \rangle_{\mathcal{W}})$ is a Hilbert space; the kernel $K$ has the following form:
\begin{equation}
K(x,y)=\langle \Phi(x),\Phi(y)\rangle_{\mathcal{W}}.
\label{eq:KernelPhi}
\end{equation}
Let $\mathcal{Y}$ be an orthonormal basis of $\mathcal{W}$ define the following continuous function $F_{y} \in \mathcal{C}(\mathcal{Z})$ defined at $x\in \mathcal{Z}$:
$$F_{y}(x)=\langle\Phi(x),y\rangle_{\mathcal{W}},$$
and let:
\vskip -0.3in 
$$\Phi(\mathcal{Y})= \overline{\text{span}}\{F_{y}, y\in \mathcal{Y}\}$$
\begin{definition}[Universal feature Map]
A feature map is universal if $\Phi(\mathcal{Y})$ is dense in $\mathcal{C}(\mathcal{Z})$, for all $\mathcal{Z}$ compact subsets of $\mathcal{X}$.i.e  A feature map is universal if $\Phi(\mathcal{Y})= \mathcal{C}(\mathcal{Z})$.
\end{definition}

The following Theorem shows the relation between universality of a kernel defined by feature map and the universality of the feature map:
\begin{theorem}[\cite{micchelli:2006_uk}, Thm 4, Relation between $\mathcal{K}(\mathcal{Z})$ and $\Phi(\mathcal{Y})$ ] For kernel defined by feature maps in \eqref{eq:KernelPhi} we have $\mathcal{K}(\mathcal{Z})=\Phi(\mathcal{Y})$. 
A kernel of form \eqref{eq:KernelPhi} is universal if and only if its feature map is universal.
\end{theorem}

Hence the following Theorem $7$ from \cite{micchelli:2006_uk}:
\begin{theorem}[\cite{micchelli:2006_uk}] Let $S=\{\phi_j, j \in I\}$, where I is a countable set and $\phi_j: \mathcal{X}\to \mathbb{R}$ continuous function. Define the following kernel
$$ K(x,y)=\sum_{j\in I }\phi_j(x)\phi_j(y).$$
$K$ is universal if and only if the set of features $S$ is universal.
\end{theorem}

\subsection{Discussion of AMA versus MA }
As we already discussed the moving average of $v$ of the difference of features  means $$\Delta_t=\frac{1}{N}\sum_{i=1}^N E(x_i)-\frac{1}{N}\sum_{i=1}^N E(G(z_i,\theta_{t})) $$ between real and generated data at each time step $t$ in the gradient descent up to time $T$, can be seen as a gradient descent in an online setting on the following cost : 
$$f^*= \min_{v}\sum_{t=1}^{T}f_{t}(v)= \sum_{t=1}^T||v-\Delta_{t}||^2_2$$
Note that we are in the online setting since $\Delta_t$ is only known when $\theta_{t}$ of the generator is updated.  The sequence $v_{t}$ generated by MA (moving average) and by AMA (ADAM moving average) is the SGD updates and ADAM updates respectively applied to the cost function $f_{t}$. Hence we can bound the regret of the sequence $\{v^{\text{MA}}_{t}\}$ and $\{v^{\text{AMA}}_{t}\}$ using known results on SGD and ADAM.
Let $d$ be the dimension of the encoding $E$. For MA, using classic regret bounds for gradient descents we obtain:
\begin{align*}
R^{\text{MA}}_{T}&=\sum_{t=1}^{T}||v^{\text{MA}}_{t}-\Delta_{t}||^2_2-f^*\leq O(\sqrt{dT}).
\end{align*}
For AMA, using ADAM regrets bounds from (Reddi et al., 2018). Let us define
\begin{align*}
R^{\text{AMA}}_{T}&=\sum_{t=1}^{T}||v^{\text{AMA}}_{t}-\Delta_{t}||^2_2-f^*.
\end{align*}
We have:
\begin{multline*}
R^{\text{AMA}}_{T} \leq  O(\sqrt{T}\sum_{i=1}^d \hat{u}^{T,\frac{1}{2}}_{i})+ \cdots \\
    O\left(\sum_{i=1}^d \sqrt{\sum_{t=1}^T(\Delta_{t,i}-v^{\text{AMA}}_{t,i})^2}\right) + C
\end{multline*}
where $\hat{u}$ are defined in the  ADAM updates as moving averages of second order moments of the gradients.
The regret bound of AMA is better than MA especially if $\sum_{i=1}^d \hat{u}^{T,\frac{1}{2}}_{i} \ll d$ and 
\begin{align*}
    \sum_{i=1}^d \sqrt{\sum_{t=1}^T(\Delta_{t,i}-v^{\text{AMA}}_{t,i})^2}&\ll\sqrt{Td}.
\end{align*}

\subsection{Mean Matching vs. Mean + Covariance Matching in GFMN}
\label{appdx:mean-vs-covariance}
In this Appendix, we present comparative results between GFMN with mean feature matching vs. GFMN with mean + covariance feature matching.
Using the first and second moments to perform feature matching gives statistical advantage over using the first moment only.
In Table \ref{tab:mean_vs_covariance},
we can see that for different feature extractors,
performing mean + covariance feature matching produces significantly better results in terms of both IS and FID.
Mroueh \etal \cite{mroueh17_mcgan} have also demonstrated the advantages of using mean + covariance matching in the context of GANs.

\begin{table*}[!ht]
\small
  \caption{CIFAR10 results for GFMN with Mean Feature Matching vs. GFMN with Mean + Covariance Feature Matching.}
  \label{tab:mean_vs_covariance}
  \centering
  \begin{tabular}{lcccc}
    \hline
     \bf Feature Extractor & \multicolumn{2}{c}{\bf Mean Matching} & \multicolumn{2}{c}{\bf Mean + Covar. Matching} \\
                                      & IS & FID (5K / 50K)               & IS & FID (5K / 50K)  \\
    \hline
    DCGAN (Encoder)            & 3.76 $\pm$ 0.04 & 96.5 / 92.5  & 4.51  $\pm$ 0.06 & 82.8 / 78.3  \\
    Resnet18                   & 7.03 $\pm$ 0.11 & 35.7 / 31.1  & 7.92 $\pm$ 0.10 & 29.1 / 24.3 \\
    VGG19                      & 7.42 $\pm$ 0.09 & 27.5 / 22.8  & 7.88 $\pm$ 0.08 & 25.5 / 20.8 \\
    \hline
  \end{tabular}
\end{table*}

\subsection{Neural Network Architectures}
In Tables \ref{arch_gen} and \ref{arch_resnet}, and Figure \ref{arch_resblock} we detail the neural net architectures used in our experiments.
In both DCGAN-like generator and discriminator, 
an extra layer is added when using images of size 64$\times$64. 
In VGG19 architecture, 
after each convolution, we apply batch normalization and ReLU.
The Resnet generator is used for CelebA$_{128\times128}$ experiments and also for some experiments with CIFAR10 and STL10.
For these two last datasets,
the Resnet generator has 3 ResBlocks only,
and the output size of the $DENSE$ layer is $4 \times 4 \times 512 $.

\begin{table}[!ht]
\caption{DCGAN like Generator}
\vskip 0.15in
\begin{center}
\begin{small}
\begin{sc}
    \begin{tabular}{|c|}
    \hline
        $z \in \mathbb{R}^{100} \sim \mathcal{N}(0,I)$ \\ \hline
        dense $\to 4\times 4 \times 512$   \\ \hline
        $4\times4$, stride=2 Deconv BN 256 ReLU \\ \hline
        $4\times4$, stride=2 Deconv BN 128 ReLU \\ \hline
        $4\times4$, stride=2 Deconv BN 64 ReLU \\ \hline
        $3\times3$, stride=1 Conv 3 BN 64 ReLU \\ \hline
        $3\times3$, stride=1 Conv 3 BN 64 ReLU \\ \hline
        $3\times3$, stride=1 Conv 3 Tanh \\ \hline
    \end{tabular}
\end{sc}
\end{small}
\end{center}
\label{arch_gen}
\end{table}

\begin{figure}[ht]
\vskip 0.2in
\begin{center}
\centerline{\includegraphics[width=0.3\columnwidth]{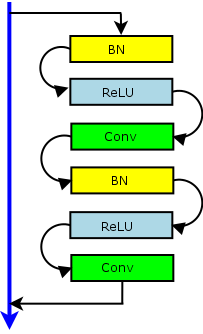}}
\caption{ResBlock}
\label{arch_resblock}
\end{center}
\end{figure}

\begin{table}[!ht]
\caption{Resnet Generator}
\label{sample-table}
\vskip 0.15in
\begin{center}
\begin{small}
\begin{sc}
    \begin{tabular}{|c|}
        \hline
        $z \in \mathbb{R}^{100} \sim \mathcal{N}(0,I)$ \\ \hline
        dense, $4\times4\times2048$ \\ \hline
        ResBlock up 1024 \\ \hline
        ResBlock up 512 \\ \hline
        ResBlock up 256 \\ \hline
        ResBlock up 128 \\ \hline
        ResBlock up 164 \\ \hline
        BN, ReLU, $3\times3$ conv $3$ \\ \hline
        Tanh \\ \hline
    \end{tabular}
\end{sc}
\end{small}
\end{center}
\label{arch_resnet}
\end{table}

\subsection{Pretraining of ImageNet Classifiers and Autoencoders}
\label{pretrain_class}
Both VGG19 and Resnet18 networks are trained with SGD with fixed $10^{-1}$ learning rate, 0.9 momentum term, and weight decay set to $5\times10^{-4}$. 
We pick models with best \emph{top-1} accuracy on the validation set over 100 epochs of training; 
29.14\% for VGG19 (image size 32$\times$32), 
and 39.63\% for Resnet18 (image size 32$\times$32). 
When training the classifiers we use random cropping and random horizontal flipping for data augmentation.
When using VGG19 and Resnet18 as feature extractors in GFMN,
we use features from the output of each ReLU that follows a conv. layer, for a total of 16 layers for VGG and 17 for Resnet18.

In our experiments with autoencoders (AE) we pretrained them using either mean squared error (MSE) or the Laplacian pyramid loss \cite{ling2006LapLoss,bojanowski2018optimizing}.
Let $E$ and $D$ be the encoder and the decoder networks with parameters $\phi$ and $\psi$,  respectively.
$$\min_{\phi,\psi} \mathbb{E}_{p_{data}}||x-D(E(x;\phi);\psi) ||^2$$
or the Laplacian pyramid loss \cite{ling2006LapLoss}
\begin{equation*}
    \label{eq:laploss}
	\text{Lap}_1(x,x') = \sum_{j} 2^{-2j} |L^j(x) - L^j(x')|_1
\end{equation*}
where~$L^j(x)$ is the~$j$-th level of the Laplacian pyramid representation of~$x$.
The Laplacian pyramid loss provides better signal for learning high frequencies of images and overcome some of the blurriness issue known from using a simple MSE loss.
\cite{bojanowski2018optimizing}
 recently demonstrated that the $\text{Lap}_1$ loss produces better results than $L_2$ loss for both autoencoders and generative models.

\subsection{Quantitative Evaluation Metrics}
\label{appdx:is_fid}
We evaluate our models using two quantitative metrics: Inception Score (IS) \cite{salimans2016improved} and Fr\'{e}chet Inception Distance (FID) \cite{heuselRUNKH17FID}.
We followed the same procedure used in previous work to calculate IS \cite{salimans2016improved,miyato2018spectral,RavuriMRV18}.  
For each trained generator, we calculate the IS for randomly generated 5000 images and repeat this procedure 10 times (for a total of 50K generated images) and report the average and the standard deviation of the IS.

We compute FID using two sample sizes of generated images: 5K and 50K.
In order to be consistent with previous works \cite{miyato2018spectral,RavuriMRV18} and be able to directly compare our quantitative results with theirs, 
the FID is computed as follows: 
\begin{itemize}
    \item CIFAR10: the statistics for the real data are computed using the 50K training images. This (real data) statistics are used in the FID computation of both 5K and 50K samples of generated images. This is consistent with both Miyato \etal \cite{miyato2018spectral} and Ravuri \etal \cite{RavuriMRV18} procedure to compute FID for CIFAR10 experiments.
    \item STL10: when using 5K generated images, the statistics for the real data are computed using the set of 5K (labeled) training images. This is consistent with the FID computation of Miyato \etal \cite{miyato2018spectral}. When using 50K generated images, the statistics for the real data are computed using a set of 50K images randomly sampled from the unlabeled STL10 dataset.
\end{itemize}
FID computation is repeated 3 times and the average is reported. There is very small variance in the FID results.

\subsection{Impact of the number of layers used for feature extraction}
\begin{figure*}[!h]
\vskip 0.2in
\begin{center}
    \begin{subfigure}[b]{0.3\textwidth}
        \includegraphics[width=\textwidth]  {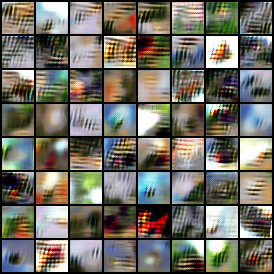}
        \caption{1 Layer}
        \label{fig:cifar10_vgg_l1}
    \end{subfigure}
    \begin{subfigure}[b]{0.3\textwidth}
        \includegraphics[width=\textwidth]  {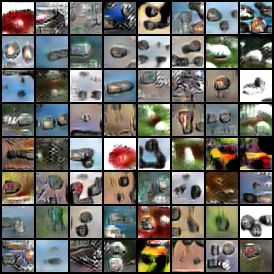}
        \caption{3 Layers}
        \label{fig:cifar10_vgg_l3}
    \end{subfigure}
    \begin{subfigure}[b]{0.3\textwidth}
        \includegraphics[width=\textwidth]  {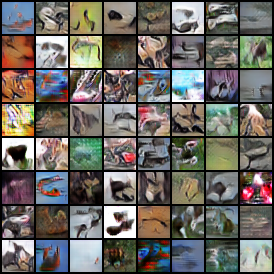}
        \caption{5 Layers}
        \label{fig:cifar10_vgg_l5}
    \end{subfigure}

    \begin{subfigure}[b]{0.3\textwidth}
        \includegraphics[width=\textwidth]  {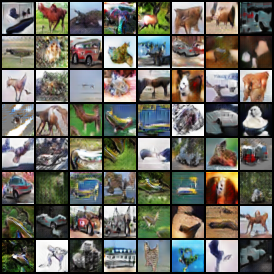}
        \caption{7 Layers}
        \label{fig:cifar10_vgg_l7}
    \end{subfigure}
    \begin{subfigure}[b]{0.3\textwidth}
        \includegraphics[width=\textwidth]  {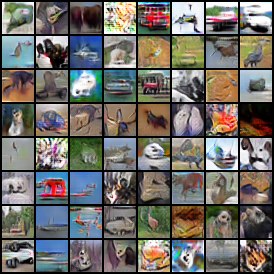}
        \caption{9 Layers}
        \label{fig:cifar10_vgg_l9}
    \end{subfigure}
    \begin{subfigure}[b]{0.3\textwidth}
        \includegraphics[width=\textwidth]  {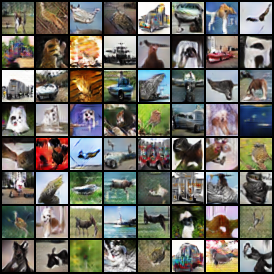}
        \caption{11 Layers}
        \label{fig:cifar10_vgg_l11}
    \end{subfigure}

    \begin{subfigure}[b]{0.3\textwidth}
        \includegraphics[width=\textwidth]  {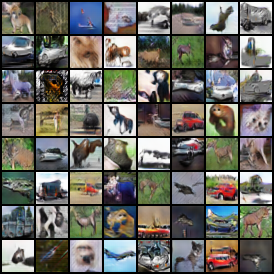}
        \caption{13 Layers}
        \label{fig:cifar10_vgg_l13}
    \end{subfigure}
    \begin{subfigure}[b]{0.3\textwidth}
        \includegraphics[width=\textwidth]  {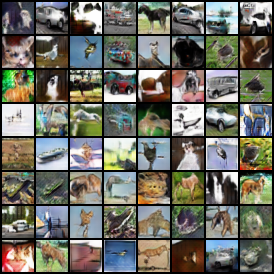}
        \caption{15 Layers}
        \label{fig:cifar10_vgg_l15}
    \end{subfigure}
    \begin{subfigure}[b]{0.3\textwidth}
        \includegraphics[width=\textwidth]  {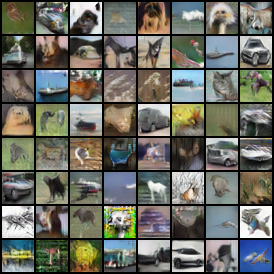}
        \caption{16 Layers}
        \label{fig:cifar10_vgg_l16}
    \end{subfigure}
\caption{Generated images from GFMN trained with a different number of VGG19 layers for feature extraction.}
\label{layers}
\end{center}
\end{figure*}

Figure \ref{layers} shows generated images from generators that were trained with a different number of layers employed to feature matching.
In all the results in Fig.\ref{layers}, the VGG19 network was used to perform feature extraction.
We can see a significant improvement in image quality when more layers are used.
Better results are achieved when 11 or more layers are used,
which corroborates the quantitative results in Sec. 5.2.

\subsection{Pretrained Generator/Discriminator in  WGAN-GP}
The objective of the experiments presented in this section is to evaluate if WGAN-GP can benefit from DCNN classifiers pretrained on ImageNet.
In the experiments, we used a WGAN-GP architecture where: 
(1) the discriminator is a VGG19 or a Resnet18; 
(2) the discriminator is pretrained on ImageNet; 
(3) the generator is pretrained on CIFAR10 through autoencoding.
Although we tried different hyperparameter combinations,
we were not able to successfully train WGAN-GP with VGG19 or Resnet18 discriminators.
Indeed, the discriminator, being pretrained on ImageNet, can quickly learn to distinguish between real and fake images. This limits the reliability of the gradient information from the discriminator, which in turn renders the training of a proper generator extremely challenging or even impossible. This is a well-known issue with GAN training \cite{goodfellow2014generative} where the training of the generator and discriminator must strike a balance. This phenomenon is covered in \cite{arjovsky17a} Section~3 (illustrated in their Figure~2) as one motivation for work like Wassertein GANs. If a discriminator can distinguish perfectly between real and fake early on, the generator cannot learn properly and the min/max game becomes unbalanced, having no good discriminator gradients for the generator to learn from, producing degenerate models. 
Figure \ref{wgangp} shows some examples of images generated by the unsuccessfully trained models.

\begin{figure}[ht]
\vskip 0.2in
\begin{center}
  \includegraphics[width=0.6\columnwidth]{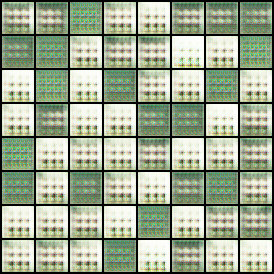}
\caption{Generated images by WGAN-GP with pretrained VGG19 as a discriminator.}
\label{wgangp}
\end{center}
\vskip -0.2in
\end{figure}

\subsection{Impact of Adam Moving Average for VGG19 feature extractor.}
\label{appdx:am_vs_ama_vgg19}

In this appendix, 
we present a comparison between the simple moving average (MA) and ADAM moving average (AMA) 
for the case where VGG19 ImageNet
classifier is used as a feature extractor. 
This experiment uses a minibatch size of 64. 
We can see
in Fig. \ref{am_vs_ama_vgg19} that AMA has a very positive effect in the quality of generated images.
GFMN trained with MA produces various images with some sort of crossing line artifacts.    

\begin{figure}[!ht]
\vskip 0.2in
\begin{center}
    \begin{subfigure}[b]{0.23\textwidth}
        \includegraphics[width=\textwidth]  {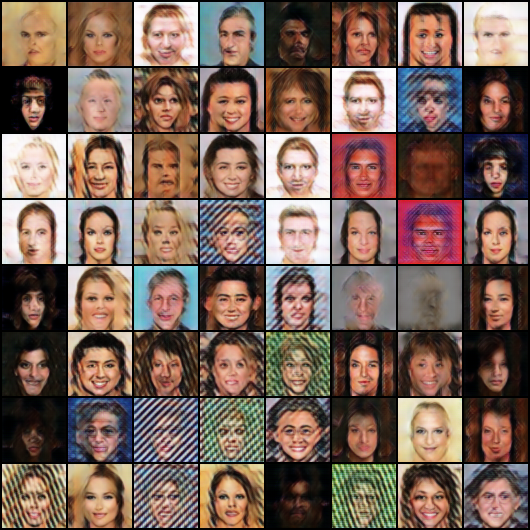}
        \caption{MA}
        \label{fig:celeba_ma_vgg19}
    \end{subfigure}
    \begin{subfigure}[b]{0.23\textwidth}
        \includegraphics[width=\textwidth]  {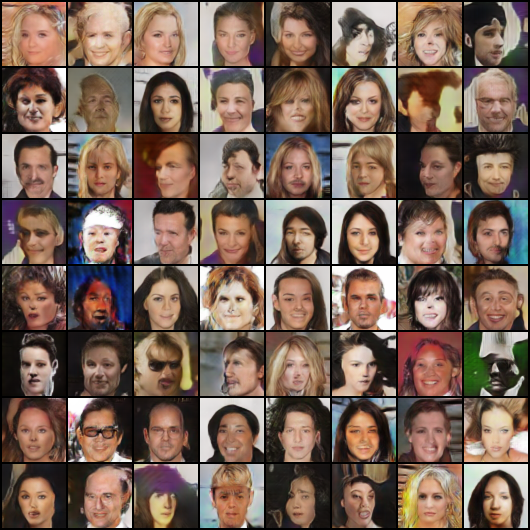}
        \caption{AMA}
        \label{fig:celeba_ama_vgg19}
    \end{subfigure}
\caption{Generated images from GFMN trained with either simple moving average (MA) or Adam
moving average (AMA). VGG19 ImageNet classifier is used as feature extractor.}
\label{am_vs_ama_vgg19}
\end{center}
\vskip -0.2in
\end{figure}

\subsection{Visual Comparison between GFMN and GMMN Generated Images.}
\label{appdx:gmmn}
Figure \ref{fig:gmmn} shows a visual comparison between images generated by GFMN (Figs. \ref{fig:cifar_vgg_gmmncomp} and \ref{fig:cifar_resnet_gmmncomp}) and Generative Moment Matching Networks (GMMN) (Figs. \ref{fig:cifar_gmmn_data} and \ref{fig:cifar_gmmn_ae}). GMMN \cite{li:2015:GMMN} generated images were obtained from Li \etal \cite{li:nips2017:MMDGAN}.
In this experiment,
both GMMN and GFMN use a DCGAN-like architecture in the generator.
Images generated by GFMN have significantly better quality compared to the ones generated by GMMN,
which corroborates the quantitative results in Sec. 5.4.

\begin{figure}[!ht]
\vskip 0.2in
\begin{center}
    \begin{subfigure}[b]{0.23\textwidth}
        \includegraphics[width=\textwidth]  {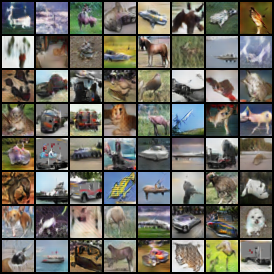}
        \caption{GFMN with VGG19 features}
        \label{fig:cifar_vgg_gmmncomp}
    \end{subfigure}
    \begin{subfigure}[b]{0.23\textwidth}
        \includegraphics[width=\textwidth]  {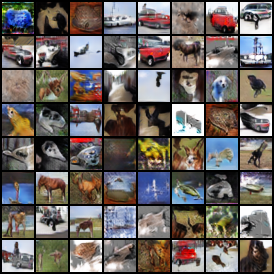}
        \caption{GFMN with Resnet18 features}
        \label{fig:cifar_resnet_gmmncomp}
    \end{subfigure}

    \begin{subfigure}[b]{0.23\textwidth}
        \includegraphics[width=\textwidth]  {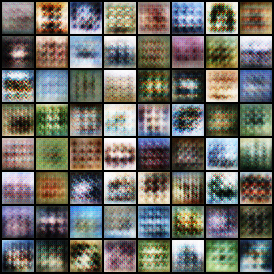}
        \caption{GMMN - Matching on data space}
        \label{fig:cifar_gmmn_data}
    \end{subfigure}
    \begin{subfigure}[b]{0.23\textwidth}
        \includegraphics[width=\textwidth]  {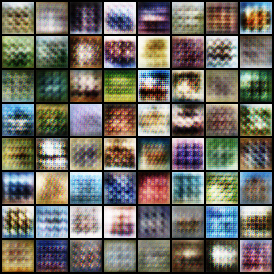}
        \caption{GMMN+AE - Matching on AE space}
        \label{fig:cifar_gmmn_ae}
    \end{subfigure}
\caption{Generated images from GFMN (\ref{fig:cifar_vgg_gmmncomp} and \ref{fig:cifar_resnet_gmmncomp}) and GMMN (\ref{fig:cifar_gmmn_data} and \ref{fig:cifar_gmmn_ae}). GMMN images were obtained from Li \etal \cite{li:nips2017:MMDGAN}.}
\label{fig:gmmn}
\end{center}
\end{figure}

\subsection{Autoencoder features vs. VGG19 features for CelebA.}
\label{appdx:ae_vs_vgg_celeba}

In this appendix, 
we present a comparison in image quality for autoencoder features vs. VGG19 features for the CelebA dataset.
We show results for both simple moving average (MA) and ADAM moving average (AMA),
for both cases we use a minibatch size of 64. 
In Fig. \ref{fig:ae_vs_vgg19},
we show generated images from GFMN trained with either
VGG19 features (top row) or autoencoder (AE) features (bottom row).
We show images generated by GFMN models trained with  simple moving average (MA) and Adam
moving average (AMA). 
We can note in the images that, 
although VGG19 features are from a cross-domain classifier, they lead to much better generation quality than AE features, 
specially for the MA case.

\begin{figure}[!ht]
\vskip 0.2in
\begin{center}
    \begin{subfigure}[b]{0.23\textwidth}
        \includegraphics[width=\textwidth]  {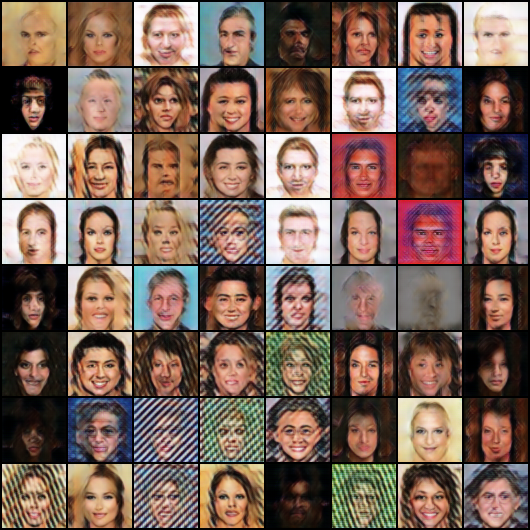}
        \caption{MA - VGG19 Features}
        \label{fig:celeba_vgg19_nopretrain_no_adammovavrg1}
    \end{subfigure}
    \begin{subfigure}[b]{0.23\textwidth}
        \includegraphics[width=\textwidth]{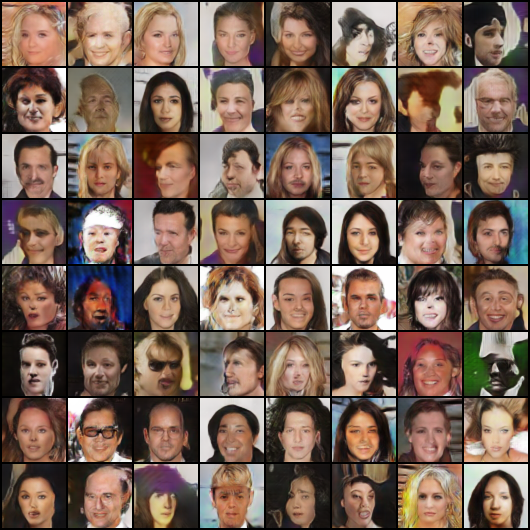}
        \caption{AMA - VGG19 Features}
        \label{fig:celeba_vgg19_nopretrain_adammovavrg1}
    \end{subfigure}

    \begin{subfigure}[b]{0.23\textwidth}
        \includegraphics[width=\textwidth]{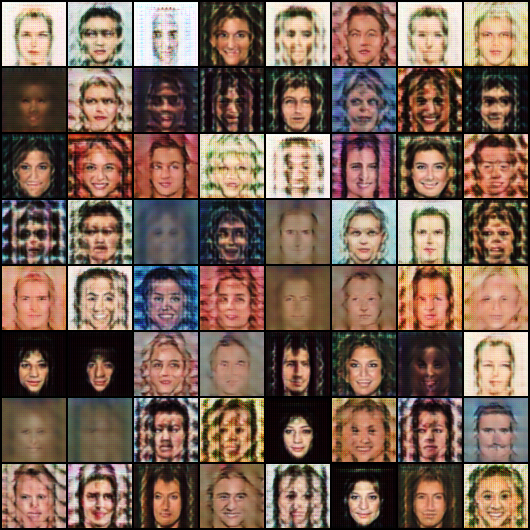}
        \caption{MA - AE Features}
        \label{fig:celeba_ma_bts641}
    \end{subfigure}
    \begin{subfigure}[b]{0.23\textwidth}
        \includegraphics[width=\textwidth]{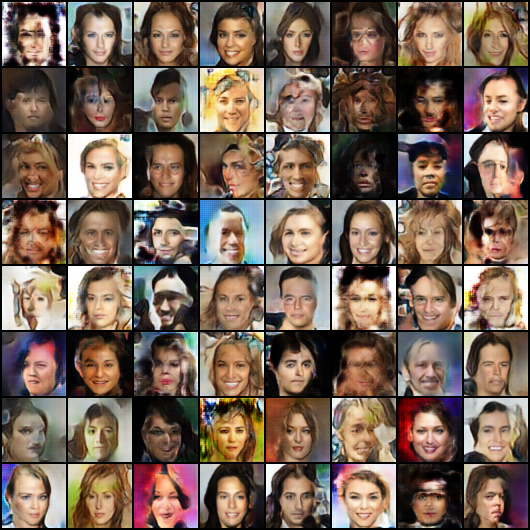}
        \caption{AMA - AE Features}
        \label{fig:celeba_ama_bts641}
    \end{subfigure}
\caption{Generated images from GFMN trained with either
    VGG19 features (top row) or autoencoder (AE) features (bottom row).
    We show images generated by GFMN models trained with  simple moving average (MA) and Adam
moving average (AMA). Although VGG19 features are from a cross-domain classifier, they perform much better than AE features, specially for the MA case.}
\label{fig:ae_vs_vgg19}
\end{center}
\end{figure}